\documentclass[11pt]{article}

\usepackage{amsmath,amssymb,amsfonts,amsthm}
\usepackage[margin=1in]{geometry}

\usepackage{xcolor}
\usepackage{algorithm}
\usepackage[noend]{algpseudocode}
\usepackage{float}

\DeclareMathOperator*{\argmin}{arg\,min}
\DeclareMathOperator*{\argmax}{arg\,max}

\newtheorem{theorem}{Theorem}
\newtheorem{remark}{Remark}

\newtheorem{definition}{Definition}

\usepackage[authoryear,sort]{natbib}

 \title{Functional Natural Policy Gradients}

\author{Aur\'elien Bibaut \and Houssam Zenati \and Thibaud Rahier \and Nathan Kallus}

\begin{document}

\maketitle

\begin{abstract}
We propose a cross-fitted debiasing device for policy learning from offline data. A key consequence of the resulting learning principle is $\sqrt N$ regret even for policy classes with complexity greater than Donsker, provided a product-of-errors nuisance remainder is $O(N^{-1/2})$. The regret bound factors into a plug-in policy error factor governed by policy-class complexity and an environment nuisance factor governed by the complexity of the environment dynamics, making explicit how one may be traded against the other.
\end{abstract}

\section{Introduction}

Personalized decision policies are increasingly central in areas such as healthcare \citep{bertsimas2017personalized}, education \citep{mandel2014offline}, and public policy \citep{kube2019allocating}, where tailoring actions to individual characteristics can improve outcomes. In many of these settings, however, actively experimenting with new policies to generate ``online data" is expensive, risky, or infeasible, which motivates methods that can evaluate and optimize policies using pre-existing ``offline data."

A variety of work studies semiparametric efficient estimation of the value of a fixed policy from offline data \citep{scharfstein1999adjusting,chernozhukov2018double,dudik2011doubly,jiang2016doubly,kallus2020double,kallus2022efficiently,kallus2022doubly}. And, a variety of work considers selecting the policy that optimizes such estimates over policies in a given class \citep{athey2021policy,zhang2013robust,kallus2021more,zhou2023offline,foster2023orthogonal,chernozhukov2019semiparametric}, which generally yields rates the scale with policy class complexity, e.g., $O_P(N^{-1/2})$ for VC classes. \cite{luedtke2020performance} get regret acceleration to $o_P(N^{-1/2})$ by leveraging an equicontinuity argument. Margin (or, noise) conditions controlling the fraction of hard decision contexts can also speed up policy-search regret \citep{boucheron2005theory,hu2022fast,hu2024contextual,tsybakov2004optimal,tsybakov2005square,koltchinskii2006local}. However, with the exception of \citet{bennett2020efficient} whose assumptions imply the optimal policy solves a smooth conditional moment restriction, none of these works leverage cross-fitting or debiasing to estimate the optimal policy itself, such as done for example when targeting a low-dimensional regular functional \citep{chernozhukov2024applied}.



In this paper we show how to perform a cross-fitted policy debiasing update to an initial ERM policy fit. Specifically, we recycle the machinery of targeted maximum loss estimation  (TMLE) \citep{vanderlaanrose2011targeted,vanderlaan2016one} to find a policy within a pre-specified policy class that maximally increases population-level policy value. To make the analogy clear for readers familiar with TMLE: here, the target functional is the optimal policy value, the nuisance we debias is the policy, and the efficient score equation we solve is an estimated first-order condition for value optimality at the population level.

We demonstrate how to solve for the functional value-optimality condition via optimization along a one-dimensional policy subclass, which we construct in the same way \cite{vanderlaan2016one} construct a universal-least favorable submodel (ULFM) within a nuisance space. It turns out that, in the policy optimization setting, ULFMs are natural policy gradient (NPG) flows \citep{kakade2001natural}. In that sense, our proposal can be seen as a functional NPG ascent, and is similar in spirit to functional optimization (see \citealp{petrulionyte2024functional,elkhoury2025kernel} for example).

The key statistical enablers of our results are two-fold. First, we construct estimated NPG flows / ULFMs on one split of the data and optimize the index along the estimated flow on another split of the data, which allows us to debias policies  living in classes much larger than Donsker. Second, we leverage the TMLE principle: in policy learning, what we are after primarily isn't a good estimate of the optimal value but a good estimate of a policy that realizes it. The TMLE principle allows us to merge these two objectives in one: since a TMLE is a plug-in estimate, a plug-in policy that realizes a good estimate of the optimal value must have good regret. 


Our results of course do not violate existing minimax results on policy learning. Our theorem provides rates in terms of an empirical process term over the one-dimensional NPG flow, which is trivially $O_P(N^{-1/2})$ and a product-of-nuisance-errors remainder term. The latter makes appear an environment-nuisance error factor and a policy error factor. What makes it possible to achieve root-$N$ regret rate over larger-than-Donsker policy classes is that the environment dynamics are learnable. Therefore simple environment dynamics can alleviate the learning burden of complex policy class.

\section{Setup}

We consider the contextual bandit setting. We observe i.i.d. copies of a context-action-reward triplet $(X,A,Y) \in \mathcal{X} \times [K] \times [0,1]$, for some $K \geq 1$, generated from a context density $q_X$, logging policy $\pi_b$, and a conditional density of reward given action and context $q_Y$, where densities are w.r.t. an appropriate dominating measure. Specifically
\begin{align}
X &\sim q_X, \\
A \mid X=x &\sim \pi_b(\cdot \mid x), \\
Y \mid X=x, A=a &\sim q_Y(\cdot \mid a,x).
\end{align}
For any policy $\pi$, generic pair of context and reward densities $q = (q_X, q_Y)$, and $f : \mathcal{X} \times [K] \times [0,1] \to \mathbb{R}$, $x \in \mathcal{X},~ a \in [K]$, define
\begin{align}
P_{q,\pi} f
:=&
\int q_X(x) \sum_a \pi(a \mid x) q_Y(y \mid a,x) f(x,a,y)\,dx\,dy,\\
P
:=& P_{q,\pi_b},\\
Q(a,x) :=& \int y q_Y(y \mid a, x) dy,\\
V(q,\pi)
:=&
\int q_X(x) \sum_a \pi(a \mid x) Q(a,x)\,dx
\\
=& P_{q,\pi}Y,\\
J_\lambda(q,\pi)
=&
V(q,\pi)
-
\lambda \int q_X(x) \sum_a \pi(a \mid x) \log \pi(a \mid x)\,dx.
\end{align}
Let $\Pi$ a class of policies. For any $\pi$, denote $T_\Pi(\pi)$ the tangent space of $\Pi$ at $\pi$, that is 
\begin{align}
T_\Pi(\pi)
:= \left\lbrace \left. \partial_\epsilon \log \pi_\epsilon \right|_{\epsilon = 0}: 
\left\lbrace \pi_\epsilon : \epsilon \right\rbrace \text{ regular one-dimensional submodel of } \Pi, \pi_{\epsilon = 0} = \pi \right\rbrace.
\end{align}

\section{Semiparametric natural gradient and natural gradient flow}

\begin{definition}[Semiparametric natural policy gradient] For any pair of environment densities $q=(q_X, q_Y)$ and policy $\pi \in \Pi$, we define the semiparametric natural policy gradient at $\pi$ under $q$ $G(q,\pi)$ as 
\begin{align} \label{eq:npg}
G(q,\pi)
\in
\argmin_{\phi \in T_\Pi(\pi)}
\left\{
P_{q,\pi}\,\frac{1}{2}\phi^2
-
\partial_\pi J_\lambda(q,\pi)(\pi\phi)
\right\}.
\end{align}
    
\end{definition}

\begin{remark}\label{rem:np_adv_to_sp}
In the nonparametric policy class, that is when $\Pi$ is the full simplex pointwise in $x$, the Riesz representer of the differential
\begin{align}
\phi \mapsto \partial_\pi J_\lambda(q,\pi)(\pi\phi)
\end{align}
under the inner product
\begin{align}
(\phi_1,\phi_2)\mapsto P_{q,\pi}[\phi_1\phi_2]
\end{align}
is the centered entropy-adjusted advantage
\begin{align}
A_{q,\pi}(x,a)
:=
\left\{Q(a,x)-\lambda \log \pi(a\mid x)\right\}
-
\sum_{a'}\pi(a'\mid x)\left\{Q(a',x)-\lambda \log \pi(a'\mid x)\right\}.
\end{align}
Indeed, for any score $\phi$ satisfying $\sum_a \pi(a\mid x)\phi(x,a)=0$,
\begin{align}
\partial_\pi J_\lambda(q,\pi)(\pi\phi)
=
P_{q,\pi}\left[A_{q,\pi}\phi\right].
\end{align}

For a semiparametric policy class $\Pi$, the natural gradient $G(q,\pi)$ is the Riesz representer of the same differential restricted to the tangent space $T_\Pi(\pi)$. Equivalently, $G(q,\pi)$ is the $L^2(P_{q,\pi})$-projection of $A_{q,\pi}$ onto $T_\Pi(\pi)$:
\begin{align}
G(q,\pi)
\in
\argmin_{\phi\in T_\Pi(\pi)}
P_{q,\pi}\left[(\phi-A_{q,\pi})^2\right].
\end{align}
Thus the nonparametric gradient is the centered entropy-adjusted advantage, while under a semiparametric policy class the gradient is its projection onto the tangent space of $\Pi$ at $\pi$.
\end{remark}

\begin{remark}
    In higher level terms, remark \ref{rem:np_adv_to_sp} states that, in the nonparametric setting, the ``direction'' of highest value ascent is the advantage. The right notion of ``directions'' in optimization, and functional optimization, and in particular the space in which gradients live, is the tangent space. In the semiparametric setting, some components of the advantage are orthogonal to the tangent space. The advantage is a gradient in the sense that its inner product against scores gives the the variation of the objective. In semiparametric statistics terms, any such object is a valid gradient. The canonical gradient is the minimum norm such gradient, and it is thus the one that has no orthogonal component to the tangent space. This is why, in our setting the canonical gradient, which we also refer to here as the natural gradient to match optimization terminology, is the projection of the advantage onto the tangent space.
\end{remark}

Note that for any regular one-dimensional subclass $\{\pi_\epsilon : \epsilon\} \subset \Pi$ such that $\pi_{\epsilon=0} = \pi$ and $\partial_\epsilon \log \pi_\epsilon\big|_{\epsilon=0} = \phi$,
we have
\begin{align}
\partial_\epsilon J_\lambda(q,\pi_\epsilon)\big|_{\epsilon=0}
=
P_{q,\pi}\,G(q,\pi)\cdot \phi.
\end{align}
Note that by definition, $G(q,\pi) \in T_\Pi(\pi)$, and therefore they are $\pi$-centered pointwise in $x$, that is,
$\sum_a \pi(a \mid x) G(q,\pi)(a,x)=0$ for any $x \in \mathcal{X}$. The natural policy gradient can be defined as a minimizer of a risk under the data-generating distribution $\pi$ via importance sampling weighting. Specifically, $G(q,\pi)=\widetilde G(P,\pi)$ for
\begin{align}
\widetilde G(P,\pi)
\in
\argmin_{\phi \in T_\Pi(\pi)}
\left\{
P\,\frac{1}{2}\frac{\pi}{\pi_b}\phi^2
-
\partial_\pi \widetilde J_\lambda(P,\pi)(\pi\phi)
\right\}\label{eq:IS_weighted_grad}
\end{align}
with
\begin{align}
\partial_\pi \widetilde J_\lambda(P,\pi)(\pi \phi)
&:=
P\,\frac{\pi}{\pi_b}\,\partial_\pi J_\lambda(q,\pi)(\pi\phi). \label{eq:IS_diff_J}
\end{align}

%

\begin{definition}[Natural policy gradient flow]
For any $\bar P,\bar\pi$, let $t \in [0, \infty) \mapsto \pi_t(\bar P,\bar\pi)$ be defined by the ordinary differential equation
\begin{align}\label{eq:npgf}
\left\{
\begin{aligned}
\frac{d}{dt}\log \pi_t(\bar P,\bar\pi)
&=
\widetilde G\bigl(\bar P,\pi_t(\bar P,\bar\pi)\bigr),\\
\pi_0(\bar P,\bar\pi)
&=
\bar\pi.
\end{aligned}
\right.
\tag{NPGF}
\end{align}
We refer to $t \in [0, \infty) \mapsto \pi_t(\bar P,\bar\pi)$ as the natural policy gradient flow.
\end{definition} 

Our construction mimics the least favorable uniform submodel construction in \cite{vanderlaan2016one}.

\begin{remark}
    The terminology ``natural policy gradient flow'' is literal here. The defining differential equation evolves the policy in the score coordinates $\log \pi_t$ using the Riesz representer of the value differential under the $L^2(P_{q,\pi})$ geometry. In a smooth finite-dimensional parametric model $\Pi=\{\pi_\theta:\theta\in\Theta\}$, this reduces to the usual natural-gradient picture: the score span defines the tangent space, the $L^2(P_{q,\pi_\theta})$ inner product induces the Fisher information, and the resulting update is the Fisher-preconditioned gradient direction. The present construction should therefore be read as a coordinate-free, semiparametric version of natural policy gradient \citep{kakade2001natural}.
\end{remark}

\section{Cross-fitted debiased optimal policy estimator}

We now construct an optimal policy estimator as follows: (1) we compute an initial (entropy regularized) empirical risk minimizer on a first split of the data, (2) we then construct a one-dimensional natural gradient flow starting at the initial estimator and estimating the gradients using a second split of the data, and (3) we select an index of the natural gradient flow that maximizes the (entropy regularized) value on a third split of the data.
\\

Formally, let $P_N^{-1}$, $P_N^0$, and $P_N^1$ be empirical measures generated by three splits of the data each consisting of $N$ i.i.d.\ draws from $P$. Let
\begin{align}
\hat\pi^0 &\in \argmax_{\pi \in \Pi} \widetilde J_\lambda(P_N^{-1},\pi),
\end{align}
be an initial policy estimator computed from the $-1$ split, let 
\begin{align}
t_1
:= \argmax_{t\ge 0}
\widetilde J_\lambda\bigl(P_N^1,\pi_t(P_N^0,\hat\pi^0)\bigr)
\end{align}
be the index of the policy along the gradient flow that optimizes the value estimated on the $1$ split
and let
\begin{align}\label{eq:hatpistar}
\hat\pi_\star := \pi_{t_1}(P_N^0,\hat\pi^0).
\end{align}
be the policy along the gradient flow indexed by the cross-fitted $t_1$. The full procedure is summarized in Algorithm~\ref{alg:debiased_functional_npgf}.

\begin{algorithm}[t!]
\caption{Cross-fitted debiased policy learning via functional natural gradient flow}
\label{alg:debiased_functional_npgf}
\begin{algorithmic}[1]
\Require Three independent empirical splits $P_N^{-1}, P_N^0, P_N^1$
\State Compute an initial regularized ERM policy
\[
\hat\pi^0 \in \argmax_{\pi\in\Pi} \widetilde J_\lambda(P_N^{-1},\pi).
\]
\State Starting from $\hat\pi^0$, construct the natural policy gradient flow
\[
t \mapsto \pi_t(P_N^0,\hat\pi^0)
\]
using the middle split $P_N^0$ to estimate the natural gradient field.
\State Select the index
\[
t_1 \in \argmax_{t\ge 0} \widetilde J_\lambda(P_N^1,\pi_t(P_N^0,\hat\pi^0))
\]
using the third split $P_N^1$.
\State Output the cross-fitted debiased policy
\[
\hat\pi_\star := \pi_{t_1}(P_N^0,\hat\pi^0).
\]
\end{algorithmic}
\end{algorithm}

As in targeted minimum loss estimation, we think of $\hat\pi_\star$ as a debiased policy. Define the oracle (soft-)optimal policy
\begin{equation}
\pi_\star \in \argmax_{\pi \in \Pi} \widetilde J_\lambda(P,\pi).
\end{equation}

\begin{remark}
Algorithmically, the procedure starts from a plug-in ERM policy, then replaces the difficult task of optimizing over the full policy class by a one-dimensional search along an estimated natural-gradient flow. This algorithm should be read as a one-shot offline policy-improvement procedure under fixed logging, not as a repeated online policy-iteration scheme. The role of cross-fitting is twofold: it makes the path construction and the path selection statistically separable, and it is precisely what allows the final regret decomposition to feature a one-dimensional empirical-process term together with product-of-errors nuisance remainders.
\end{remark}

 The following theorem provides conditions under which a regret bound holds.

\begin{theorem}\label{thm:regret_bound_decomp}
Assume $\Pi$ is convex and that $t_1$ is an interior maximizer of
\begin{align}
t \mapsto \widetilde J_\lambda\bigl(P_N^1,\pi_t(P_N^0,\hat\pi^0)\bigr).
\end{align}
and $\hat \pi_\star$ defined as in \eqref{eq:hatpistar}. Then
\begin{align}
J_\lambda(q,\pi_\star)-J_\lambda(q,\hat\pi_\star)
\le I+II+III,
\end{align}
where
\begin{align}
I
&=
(P-P_N^1)\,\frac{\hat\pi_\star}{\pi_b}\,
\widetilde G(P_N^0,\hat\pi_\star)\cdot \left(\frac{\pi_\star}{\hat\pi_\star}-1\right),\\
II
&=
P\,\frac{\hat\pi_\star}{\pi_b}
\bigl(\widetilde G(P,\hat\pi_\star)-\widetilde G(P_N^0,\hat\pi_\star)\bigr)\cdot \left(\frac{\pi_\star}{\hat\pi_\star}-1\right),\\
III &
=
\left\|\widetilde G(P_N^0,\hat\pi_\star)-\widetilde G(P_N^1,\hat\pi_\star)\right\|_{\hat\pi_\star,P_N^1}
\left\|\frac{\pi_\star}{\hat\pi_\star}-1\right\|_{\hat\pi_\star,P_N^1},
\end{align}
where, for any $f : \mathcal{X} \times [K] \times [0,1] \to \mathbb{R}$, policy $\bar \pi$ and distribution $\bar P$ with domain $\mathcal{X} \times [K] \times [0,1]$, $\| f \|_{\bar \pi, \bar P} := (\bar P \{(\bar \pi / \pi_b) f^2 \})^{1/2}$.
\end{theorem}

\begin{remark}\label{rem:termI_discussion}
    Conditional on the $-1$ and $0$ splits, $I$ is an empirical process term over the one-dimensional gradient flow which is substantially simpler than a global empirical-process analysis over the full policy class. It is trivially $O_P(N^{-1/2})$ under a compactness condition that holds essentially for free.
\end{remark}

\begin{remark}\label{rem:termII_discusion}
    Terms $II$ and $III$ are error-product terms where the first factor in each is an environment-nuisance error factor (measuring how well the environment-dependent quantities are learned) and the second is a policy error factor (measuring how far the selected policy is from the oracle policy along the policy side). It may help to understand why the first factors are pure environment-nuisance errors to notice that under a locally fully nonparametric $\Pi$, they reduce to differences in estimated advantage functions.
\end{remark}

\begin{remark}
    Remarks \ref{rem:termI_discussion} and \ref{rem:termII_discusion} explain the core statistical implications of the decomposition of theorem \ref{thm:regret_bound_decomp}. Rather than paying policy-class complexity through a global stochastic process over $\Pi$, we pay it only through its interaction with environment estimation error, which is what makes root-$N$ regret possible beyond the classical Donsker regime.
\end{remark}

\begin{remark}
   The theorem controls soft regret in terms of the entropy-regularized value $J_\lambda$, not directly hard regret in terms of the value $V$. The two differ only through the bias induced by entropy regularization: in the finite-action case, hard regret is soft regret plus at most an $O(\lambda)$ term, more precisely $\lambda \log K$. Hence $\lambda$ trades statistical/optimization stability against distance to the unregularized optimum.
\end{remark}

\begin{remark}
    The reason why we optimize the entropy penalized value and not directly the  unpenalized value along the gradient flow is to ensure that the optimum is realized at an interior point. The interior stationarity condition ensures that the debiased policy satisfies a certain zero-gradient equation. In semiparametric statistics terms, this annulation of the gradient is referred to as satisfying the efficient influence function (EIF) equation. Finding a nuisance that satisfies the EIF equation is what targeted learning does.
\end{remark}

\begin{proof}
The proof combines the first-order condition for the cross fitted index $t_1$ with a von-Mises expansion in terms of the natural gradient.

\paragraph{Stationarity at an interior point of the gradient flow.} 
From the definitions \eqref{eq:IS_diff_J} and \eqref{eq:IS_weighted_grad} of the IS-weighted differential $\tilde J_\lambda$ and of $\tilde G$ the gradient in terms of the $P$-risk, stationarity at $G(P_N^1,\pi_t(P_N^0,\hat\pi^0))$ and the chain rule,
\begin{equation}
\frac{d}{dt}\widetilde J_\lambda(P_N^1,\pi_t(P_N^0,\hat\pi^0))
=
P_N^1 \frac{\pi_t(P_N^0,\hat\pi^0)}{\pi_b}
\widetilde G(P_N^1,\pi_t(P_N^0,\hat\pi^0))
\cdot
\frac{d}{dt}\log \pi_t(P_N^0,\hat\pi^0).
\end{equation}
Then from the definition of the gradient flow, we have
\[
\frac{d}{dt}\log \pi_t(P_N^0,\hat\pi^0)
=
\widetilde G(P_N^0,\pi_t(P_N^0,\hat\pi^0)),
\]
and therefore, from stationarity at $t_1$, it holds that
\begin{align}
0&=
P_N^1 \frac{\hat\pi_\star}{\pi_b}
\widetilde G(P_N^1,\hat\pi_\star)\cdot \widetilde G(P_N^0,\hat\pi_\star) \\
&=\left\langle \widetilde G(P_N^1,\hat\pi_\star),\widetilde G(P_N^0,\hat\pi_\star)\right\rangle_{\hat\pi_\star,P_N^1}\label{eq:stationary_condition},
\end{align}
where, for any distribution $\bar P$ with domain $\mathcal{X} \times [K] \times [0,1]$, policy $\bar \pi$ and $f_1, f_2 : \mathcal{X} \times [K] \times [0,1] \to \mathbb{R}$, we define the $\langle \cdot, \cdot \rangle_{\bar P, \bar \pi}$ inner product by
\begin{align}
    \langle f_1 , f_2 \rangle_{\bar P, \bar \pi} := \bar P \frac{\bar \pi}{\pi_b} f_1 f_2.
\end{align}
\paragraph{Submodel definition.} We introduce 
\begin{align}
\phi_\star &:= \frac{\pi_\star}{\hat\pi_\star} - 1,
\end{align}
which is a score, and we construct a one-dimensional policy class $\{ \pi_\epsilon : \epsilon \in [0,1] \}$ with score $\phi^\star$ at every $\epsilon$ by letting for every $\epsilon \in [0,1]$ 
$\pi_\epsilon := \hat\pi_\star(1+\epsilon\phi_\star).$ Convexity of $\Pi$ ensures that $\{ \pi_\epsilon : \epsilon \in [0,1] \}\subset \Pi$.
Then, a second order Taylor expansion provides the existence of $\widetilde\epsilon \in [0,1]$ such that
\begin{align}
J_\lambda(q,\pi_\star)-J_\lambda(q,\hat\pi_\star)
&=
\partial_\epsilon J_\lambda(q,\pi_\epsilon)\big|_{\epsilon=0}
+
\frac{1}{2}\,\partial_{\epsilon\epsilon}J_\lambda(q,\pi_\epsilon)\big|_{\epsilon=\widetilde\epsilon}\\
&=
P\,\frac{\hat\pi_\star}{\pi_b}\,\widetilde G(P,\hat\pi_\star)\cdot \phi_\star
-
\frac{\lambda}{2}\,\partial_{\epsilon\epsilon}P\,H(\pi_\epsilon)(X)\big|_{\epsilon=\widetilde\epsilon},
\end{align}
by definition of the gradient at $\hat \pi_\star$, because the value term $V(q, \pi_\epsilon)$ is affine in $\epsilon$ and with
\begin{align}
H(\pi)(x) := \sum_a \pi(a \mid x)\log \pi(a \mid x).
\end{align}

\paragraph{Second order derivative of the entropy along the path.} We have that
\begin{align}
H(\pi_\epsilon)(x)
&=
\sum_a \hat\pi_\star(a \mid x)(1+\epsilon\phi_\star(a \mid x))
\log\!\left(\hat\pi_\star(a \mid x)(1+\epsilon\phi_\star(a \mid x))\right)\\
&= Q(\epsilon)+S(\epsilon),
\end{align}
with
\begin{align}
Q(\epsilon)
&:=
\sum_a \hat\pi_\star(a \mid x)(1+\epsilon\phi_\star(a \mid x))
\log \hat\pi_\star(a \mid x),\\
S(\epsilon)
&:=
\sum_a \hat\pi_\star(a \mid x)(1+\epsilon\phi_\star(a \mid x))
\log(1+\epsilon\phi_\star(a \mid x)).
\end{align}
We have
\begin{align}
S'(\epsilon)
&=
\sum_a \hat\pi_\star(a \mid x)\phi_\star(a \mid x)
\log(1+\epsilon\phi_\star(a \mid x))
+
\sum_a \hat\pi_\star(a \mid x)\phi_\star(a \mid x),
\end{align}
and $Q''(\epsilon)=0$.
We have
\begin{align}
S''(\epsilon)
&=
\sum_a \hat\pi_\star(a \mid x)
\frac{\phi_\star^2(a \mid x)}{1+\epsilon\phi_\star(a \mid x)}\\
&=
\sum_a \hat\pi_\star(a \mid x)
\frac{\hat\pi_\star(a \mid x)}{\pi_\epsilon(a \mid x)}\phi_\star^2(a \mid x)
>0.
\end{align}
Hence
\begin{align}
\partial_{\epsilon\epsilon}P\,H(\pi_\epsilon)(X)\big|_{\epsilon=\widetilde\epsilon}>0,
\end{align}

\paragraph{Von Mises Expansion.}
The second order derivative above gives us that
\begin{align}
J_\lambda(q,\pi_\star)-J_\lambda(q,\hat\pi_\star)
\le
P\,\frac{\hat\pi_\star}{\pi_b}\,\widetilde G(P,\hat\pi_\star)\cdot \phi_\star
= I+II+III',
\end{align}
with
\begin{align}
I
&:=
(P-P_N^1)\,\frac{\hat\pi_\star}{\pi_b}\,
\widetilde G(P_N^0,\hat\pi_\star)\cdot \phi_\star,\\
II
&:=
P\,\frac{\hat\pi_\star}{\pi_b}
\bigl(\widetilde G(P,\hat\pi_\star)-\widetilde G(P_N^0,\hat\pi_\star)\bigr)\cdot \phi_\star,\\
III'
&:=
P_N^1\,\frac{\hat\pi_\star}{\pi_b}\,
\widetilde G(P_N^0,\hat\pi_\star)\cdot \phi_\star.
\end{align}
Therefore, from Pythagoras and the orthogonality identity \eqref{eq:stationary_condition} arising from interior stationarity at $t_1$,
\begin{align}
\left\|\widetilde G(P_N^0,\hat\pi_\star)\right\|_{\hat\pi_\star,P_N^1}
\le
\left\|\widetilde G(P_N^0,\hat\pi_\star)-\widetilde G(P_N^1,\hat\pi_\star)\right\|_{\hat\pi_\star,P_N^1},
\end{align}
and then from Cauchy--Schwarz
\begin{align}
\vert III' \vert
\le
\left\|\widetilde G(P_N^0,\hat\pi_\star)-\widetilde G(P_N^1,\hat\pi_\star)\right\|_{\hat\pi_\star,P_N^1}
\left\|\phi_\star\right\|_{\hat\pi_\star,P_N^1}.
\end{align}
\end{proof}

\bibliographystyle{plainnat}   
\bibliography{refs}

@article{athey2021policy,
  author  = {Athey, Susan and Wager, Stefan},
  title   = {Policy Learning With Observational Data},
  journal = {Econometrica},
  volume  = {89},
  number  = {1},
  pages   = {133--161},
  year    = {2021},
  doi     = {10.3982/ECTA15732}
}

@inproceedings{mandel2014offline,
  title     = {Offline policy evaluation across representations with applications to educational games},
  author    = {Mandel, Travis and Liu, Yun-En and Levine, Sergey and Brunskill, Emma and Popovic, Zoran},
  booktitle = {AAMAS},
  volume    = {1077},
  year      = {2014}
}

@inproceedings{kube2019allocating,
  title     = {Allocating interventions based on predicted outcomes: A case study on homelessness services},
  author    = {Kube, Amanda and Das, Sanmay and Fowler, Patrick J.},
  booktitle = {Proceedings of the AAAI Conference on Artificial Intelligence},
  volume    = {33},
  number    = {01},
  pages     = {622--629},
  year      = {2019}
}

@article{bertsimas2017personalized,
  title     = {Personalized diabetes management using electronic medical records},
  author    = {Bertsimas, Dimitris and Kallus, Nathan and Weinstein, Alexander M. and Zhuo, Ying Daisy},
  journal   = {Diabetes Care},
  volume    = {40},
  number    = {2},
  pages     = {210--217},
  year      = {2017},
  publisher = {American Diabetes Association}
}

@article{luedtke2020performance,
  author  = {Luedtke, Alex and Chambaz, Antoine},
  title   = {Performance Guarantees for Policy Learning},
  journal = {Annales de l'Institut Henri Poincar{\'e}, Probabilit{\'e}s et Statistiques},
  volume  = {56},
  number  = {3},
  pages   = {2162--2188},
  year    = {2020},
  doi     = {10.1214/19-AIHP1034}
}

@article{chernozhukov2018double,
  author  = {Chernozhukov, Victor and Chetverikov, Denis and Demirer, Mert and Duflo, Esther and Hansen, Christian and Newey, Whitney and Robins, James},
  title   = {Double/Debiased Machine Learning for Treatment and Structural Parameters},
  journal = {The Econometrics Journal},
  volume  = {21},
  number  = {1},
  pages   = {C1--C68},
  year    = {2018},
  doi     = {10.1111/ectj.12097}
}

@book{vanderlaanrose2011targeted,
  author    = {van der Laan, Mark J. and Rose, Sherri},
  title     = {Targeted Learning: Causal Inference for Observational and Experimental Data},
  publisher = {Springer},
  year      = {2011}
}

@article{vanderlaan2016one,
  author    = {van der Laan, Mark and Gruber, Susan},
  title     = {One-Step Targeted Minimum Loss-based Estimation Based on Universal Least Favorable One-Dimensional Submodels},
  journal   = {The International Journal of Biostatistics},
  volume    = {12},
  number    = {1},
  pages     = {351--378},
  year      = {2016},
  doi       = {10.1515/ijb-2015-0054}
}

@inproceedings{kakade2001natural,
  author    = {Kakade, Sham M.},
  title     = {A Natural Policy Gradient},
  booktitle = {Advances in Neural Information Processing Systems 14},
  year      = {2001}
}

@inproceedings{chernozhukov2019semiparametric,
  author    = {Chernozhukov, Victor and Demirer, Mert and Lewis, Greg and Syrgkanis, Vasilis},
  title     = {Semi-Parametric Efficient Policy Learning with Continuous Actions},
  booktitle = {Advances in Neural Information Processing Systems 32},
  year      = {2019}
}

@article{petrulionyte2024functional,
  title   = {Functional bilevel optimization for machine learning},
  author  = {Petrulionyte, Ieva and Mairal, Julien and Arbel, Michael},
  journal = {Advances in Neural Information Processing Systems},
  volume  = {37},
  pages   = {14016--14065},
  year    = {2024}
}

@article{elkhoury2025kernel,
  author  = {El Khoury, Fares and Pauwels, Edouard and Vaiter, Samuel and Arbel, Michael},
  title   = {Learning Theory for Kernel Bilevel Optimization},
  journal = {arXiv preprint arXiv:2502.08457},
  year    = {2025}
}

@inproceedings{dudik2011doubly,
  title     = {Doubly Robust Policy Evaluation and Learning},
  author    = {Dud{\'\i}k, Miroslav and Langford, John and Li, Lihong},
  booktitle = {Proceedings of the 28th International Conference on Machine Learning},
  year      = {2011}
}

@article{boucheron2005theory,
  title     = {Theory of classification: A survey of some recent advances},
  author    = {Boucheron, St{\'e}phane and Bousquet, Olivier and Lugosi, G{\'a}bor},
  journal   = {ESAIM: Probability and Statistics},
  volume    = {9},
  pages     = {323--375},
  year      = {2005},
  publisher = {EDP Sciences}
}

@article{hu2022fast,
  title     = {Fast rates for contextual linear optimization},
  author    = {Hu, Yichun and Kallus, Nathan and Mao, Xiaojie},
  journal   = {Management Science},
  volume    = {68},
  number    = {6},
  pages     = {4236--4245},
  year      = {2022},
  publisher = {INFORMS}
}

@article{hu2024contextual,
  title   = {Contextual Linear Optimization with Partial Feedback},
  author  = {Hu, Yichun and Kallus, Nathan and Mao, Xiaojie and Wu, Yanchen},
  journal = {arXiv preprint arXiv:2405.16564},
  year    = {2024}
}

@article{tsybakov2004optimal,
  title     = {Optimal aggregation of classifiers in statistical learning},
  author    = {Tsybakov, Alexander B.},
  journal   = {The Annals of Statistics},
  volume    = {32},
  number    = {1},
  pages     = {135--166},
  year      = {2004},
  publisher = {Institute of Mathematical Statistics}
}

@article{tsybakov2005square,
  title  = {Square root penalty: adaptation to the margin in classification and in edge estimation},
  author = {Tsybakov, Alexandre B. and van de Geer, Sara A.},
  year   = {2005}
}

@article{koltchinskii2006local,
  title  = {Local Rademacher Complexities and Oracle Inequalities in Risk Minimization},
  author = {Koltchinskii, Vladimir},
  year   = {2006}
}

@article{kallus2022efficiently,
  title     = {Efficiently breaking the curse of horizon in off-policy evaluation with double reinforcement learning},
  author    = {Kallus, Nathan and Uehara, Masatoshi},
  journal   = {Operations Research},
  volume    = {70},
  number    = {6},
  pages     = {3282--3302},
  year      = {2022},
  publisher = {INFORMS}
}

@article{kallus2020double,
  title   = {Double reinforcement learning for efficient off-policy evaluation in Markov decision processes},
  author  = {Kallus, Nathan and Uehara, Masatoshi},
  journal = {Journal of Machine Learning Research},
  volume  = {21},
  number  = {167},
  pages   = {1--63},
  year    = {2020}
}

@article{zhang2013robust,
  title   = {Robust estimation of optimal dynamic treatment regimes for sequential treatment decisions},
  author  = {Zhang, Baqun and Tsiatis, Anastasios A. and Laber, Eric B. and Davidian, Marie},
  journal = {Biometrika},
  volume  = {100},
  number  = {3},
  year    = {2013}
}

@inproceedings{jiang2016doubly,
  title     = {Doubly robust off-policy value evaluation for reinforcement learning},
  author    = {Jiang, Nan and Li, Lihong},
  booktitle = {International Conference on Machine Learning},
  pages     = {652--661},
  year      = {2016},
  organization = {PMLR}
}

@article{scharfstein1999adjusting,
  title     = {Adjusting for nonignorable drop-out using semiparametric nonresponse models},
  author    = {Scharfstein, Daniel O. and Rotnitzky, Andrea and Robins, James M.},
  journal   = {Journal of the American Statistical Association},
  volume    = {94},
  number    = {448},
  pages     = {1096--1120},
  year      = {1999},
  publisher = {Taylor \& Francis}
}

@article{kallus2021more,
  title     = {More efficient policy learning via optimal retargeting},
  author    = {Kallus, Nathan},
  journal   = {Journal of the American Statistical Association},
  volume    = {116},
  number    = {534},
  pages     = {646--658},
  year      = {2021},
  publisher = {Taylor \& Francis}
}

@article{foster2023orthogonal,
  title     = {Orthogonal statistical learning},
  author    = {Foster, Dylan J. and Syrgkanis, Vasilis},
  journal   = {The Annals of Statistics},
  volume    = {51},
  number    = {3},
  pages     = {879--908},
  year      = {2023},
  publisher = {Institute of Mathematical Statistics}
}

@article{zhou2023offline,
  title     = {Offline multi-action policy learning: Generalization and optimization},
  author    = {Zhou, Zhengyuan and Athey, Susan and Wager, Stefan},
  journal   = {Operations Research},
  volume    = {71},
  number    = {1},
  pages     = {148--183},
  year      = {2023},
  publisher = {INFORMS}
}

@inproceedings{kallus2022doubly,
  title     = {Doubly robust distributionally robust off-policy evaluation and learning},
  author    = {Kallus, Nathan and Mao, Xiaojie and Wang, Kaiwen and Zhou, Zhengyuan},
  booktitle = {International Conference on Machine Learning},
  pages     = {10598--10632},
  year      = {2022},
  organization = {PMLR}
}

@inproceedings{bennett2020efficient,
  title     = {Efficient policy learning from surrogate-loss classification reductions},
  author    = {Bennett, Andrew and Kallus, Nathan},
  booktitle = {International Conference on Machine Learning},
  pages     = {788--798},
  year      = {2020},
  organization = {PMLR}
}

@article{chernozhukov2024applied,
  title   = {Applied causal inference powered by ML and AI},
  author  = {Chernozhukov, Victor and Hansen, Christian and Kallus, Nathan and Spindler, Martin and Syrgkanis, Vasilis},
  journal = {arXiv preprint arXiv:2403.02467},
  year    = {2024}
}

\end{document}